\documentclass[letterpaper]{article} % DO NOT CHANGE THIS
\usepackage%[submission]
{aaai2026}  % DO NOT CHANGE THIS
\usepackage{times}  % DO NOT CHANGE THIS
\usepackage{helvet}  % DO NOT CHANGE THIS
\usepackage{courier}  % DO NOT CHANGE THIS
\usepackage[hyphens]{url}  % DO NOT CHANGE THIS
\usepackage{graphicx} % DO NOT CHANGE THIS
\usepackage{xcolor}
\usepackage{amsfonts}
\usepackage{amsmath}
\usepackage{natbib}  % DO NOT CHANGE THIS AND DO NOT ADD ANY OPTIONS TO IT
\usepackage{caption} % DO NOT CHANGE THIS AND DO NOT ADD ANY OPTIONS TO IT
\usepackage{bm}
\usepackage{stmaryrd} % Where \llbracket and \rrbracket are defined
\usepackage{algorithm}
\usepackage{algorithmic}

\usepackage{newfloat}
\usepackage{listings}
\DeclareCaptionStyle{ruled}{labelfont=normalfont,labelsep=colon,strut=off} % DO NOT CHANGE THIS
\floatstyle{ruled}
\newfloat{listing}{tb}{lst}{}
\floatname{listing}{Listing}
\nocopyright% -- Your paper will not be published if you use this command
\usepackage{tikz}
\usetikzlibrary{arrows}
\usetikzlibrary {quotes}
\tikzset{node quotes mean={label={[#2,name={#1}]#1}}}
\usepackage{graphicx} % 

\newcommand{\eff}{\operatorname{eff}}
\newcommand{\pre}{\operatorname{pre}}

\newtheorem{definition}{Definition}
\newtheorem{lemma}{Lemma}

\newtheorem{theorem}{Theorem}

\title{Intermediate Results on the Complexity of STRIPS\(_{1}^{1}\)}
\author{}
\date{}

 \author{
 Stefan Edelkamp\textsuperscript{\rm 1}, 
 Jiří Fink \textsuperscript{\rm 1},
 Petr Gregor\textsuperscript{\rm 1}, 
 Anders Jonsson\textsuperscript{\rm 3}, 
 Bernhard Nebel\textsuperscript{\rm 2} 
 }

\affiliations{
    \textsuperscript{\rm 1} School of Computer Science,
Faculty of Mathematics and Physics, Charles University, Prague, 
Czech Republic \\
 \textsuperscript{\rm 2} Artificial Intelligence and Machine Learning Group,
Pompeu Fabra University, Barcelona, Spain \\
 \textsuperscript{\rm 3} Department of Computer Science, 
 University of Freiburg,
 Georges-K\"ohler-Allee, 79110 Freiburg, Germany \\
}

\usepackage{bibentry}
\begin{document}

\maketitle

\begin{abstract} 
This paper is based on Bylander's results on the computational complexity of propositional STRIPS planning. He showed that when only ground literals are permitted, determining plan existence is PSPACE-complete even if operators are limited to two pre-conditions and two postconditions. While NP-hardness is settled, it is unknown whether propositional STRIPS with operators that only have one precondition and one effect is NP-complete.
%, or whether it generates exponentially long plans. 
We shed light on the question whether this small solution hypothesis for STRIPS\(_{1}^{1}\) is true, calling a SAT solver for small instances, introducing the literal graph, and mapping it to Petri nets.
\end{abstract}

\section{Introduction}

STRIPS \cite{DBLP:journals/ai/FikesN71}, the acronym of the Stanford Research Institute Planning System, is one the earliest formalisms for action planning. The STRIPS planning language has greatly influenced later developments of the planning domain description language PDDL as proposed by McDermott (\citeyear{DBLP:journals/aim/McDermott00}), Fox and Long~(\citeyear{DBLP:journals/jair/FoxL03}) or Hoffmann and Edelkamp (\citeyear{DBLP:conf/gg/EdelkampJL06}), and has been compared to other finite-state variable planning such as SAS$^+$ \cite{DBLP:conf/ijcai/BackstromN93,DBLP:conf/ki/Nebel99,DBLP:journals/jair/Helmert06}.

Bylander~(\citeyear{DBLP:journals/ai/Bylander94}) studied the complexity of STRIPS planning and showed that the problem of determining plan existence is PSPACE-complete, meaning that the decision problem can be solved using polynomial memory, and that every PSPACE-complete problem polynomially reduces to it. Bylander's reduction proof uses an encoding of a Turing machine with STRIPS operators. PSPACE containment is derived through a divide-and-conquer planning strategy.
An important research question is to establish fragments of STRIPS planning that have lower complexity, i.e., 
being polynomial or NP-complete. 
%Although it is currently unknown whether P $\neq$ NP is in fact (arguably the most famous unsolved problem in computer science), it seems highly unlikely that the polynomial hierarchy collapses, in which case such fragments are indeed easier to solve than general STRIPS planning. Even when not directly applicable to a given STRIPS planning problem, tractable fragments are commonly used to compute heuristics that guide search~\cite{DBLP:journals/jair/HoffmannN01,Katz08}.

One famous fragment of STRIPS planning is  delete relaxation~\cite{DBLP:journals/jair/HoffmannN01}, which is known to be NP-complete. Bylander took a different route and asked what happens if the operators of a STRIPS problem are limited to at most $p$ preconditions and $e$ effects, giving rise to fragments \textsc{STRIPS}$_p^e$. Concretely, Bylander showed that STRIPS$_1^1$ is NP-hard. This follows from his results for the NP-completeness of STRIPS$_{1}^{1+}$ (e.g., the effect must be positive), and since that's a special case of STRIPS$_1^1$ NP-hardness follows. Later on, Jonsson et al.~(\citeyear{jonsson2014limitations}) proved that \textsc{STRIPS}$_2^1$ is PSPACE-complete. In Chapter 5 they wrote: \emph{Bylander showed that the problem of plan existence is PSPACE-complete for STRIPS planning instances whose actions have one postcondition and unbounded number of preconditions (either positive or negative). He did not prove a result for a bounded number $k$ of preconditions, but conjectured that plan existence falls into the polynomial hierarchy in a regular way, with the precise complexity determined by $k$. In this section we modify our previous reduction such that actions have at most two preconditions, thus showing that plan existence is
PSPACE-complete for $k=2$ and proving Bylander's conjecture to be wrong.}. However, the exact complexity of STRIPS$_1^1$ remains unknown, i.e.~there is no proof of containment in NP nor for PSPACE-hardness.

In this paper, we consider the following question for STRIPS$^1_1$: Given a set of $n$ propositions (facts, atoms) is it possible to find a set of STRIPS actions of at most one precondition and one effect that generate shortest plans of exponential size in $n$?
%This is an open research problem, and 
This paper attempts to shed new light on this question. 
%The implications are as follows: 
%\begin{itemize}
%  \item 
If there are exponentially long shortest plans, we would need exponential time generating them. While such EXP-TIME problems do not contradict polynomial space (as it was shown with the general STRIPS problem), by definition we could get closer to the assertion that the problem was not in NP. However,
  from an exponential lower bound on the plan length, we cannot directly infer that the problem was not in NP, as there might be other 
  polynomial certificates. 
%  \item
If at least one plan has polynomial length, then we can verify it in polynomial time, implying containment in NP.
%\end{itemize}
%
We aim at closing the gap from different sides: this paper gives proof insights and an AI assisted proof attempt, it introduces the literal graph and overlap tree with examples and observation on plan finding there. We recall a scaling system that leads to plans quadratic in the solution length, we prove with a SAT solver the maximum plan length for $n=5$ and $n=6$ is smaller than $2^n$, we give mathematical derivations indicating that this will always be the case for larger values of $n$, we provide an LLM generated proof attempt and compilation schemes to conservative safe Petri nets. and cooperative directed MAPF.

\section{STRIPS\(_{1}^{1}\)}

Let \(V\) be a set of propositional variables. A \emph{literal} \(l\) is a positive or negative variable in \(V\), i.e. \(l=v\) or \(l=\bar{v}\) for some \(v\in V\). A set of literals \(L\) is well-defined if and only if \(v\notin L\) or \(\bar{v}\notin L\) for each \(v\in V\). A \emph{state} \(s\) is a well-defined set of literals such that \(|s|=|V|\). A set of literals \(L\) \emph{holds} in a state \(s\) if and only if \(L\subseteq s\). The result of applying a well-defined set of literals \(L\) to a state \(s\) is a new state \(s\oplus L=(s\setminus\bar{L})\cup L\), where \(\bar{L}=\{\bar{l} \mid l\in L\}\) contains the negation of all literals in \(L\). Given a set of propositional variables \(V\), let \(L(V)= V \cup \bar{V}\) be the set of all literals on \(V\).

A {STRIPS planning problem} is a tuple \(\bm{d}=\langle V,A\rangle\) where \(V\) is a set of propositional variables and \(A\) is a set of actions. Each action \(a=\langle\mathrm{pre}(a),\mathrm{eff}(a)\rangle\in A\) has precondition \(\mathrm{pre}(a)\) and effect \(\mathrm{eff}(a)\), both well-defined sets of literals in \(L(V)\). Action \(a\) is applicable in state \(s\) if and only if \(\mathrm{pre}(a)\) holds in \(s\) and results in a new state \(s\oplus\mathrm{eff}(a)\). Two states \(s_{I}\) and \(s_{G}\) induce a planning instance \(\bm{p}=\langle V,A,s_{I},s_{G}\rangle\). A plan for \(\bm{p}\) is a sequence of actions \(\omega=\langle a_{1},\ldots,a_{k}\rangle\) such that \(\mathrm{pre}(a_{1})\) holds in \(s_{I}\) and, for each \(1<i\leq k\), \(\mathrm{pre}(a_{i})\) holds in \(s\oplus\mathrm{eff}(a_{1})\oplus\cdots\oplus\mathrm{eff}(a_{i-1})\). The plan \(\omega\) solves \(\bm{p}\) if and only if \(s\oplus\mathrm{eff}(a_{1})\oplus\cdots\oplus\mathrm{eff}(a_{k})=s_{G}\). Note that we only consider fully specified goal states \(s_{G}\). Given \(\omega\), a subplan of \(\omega\) is a subsequence of actions in \(\omega\).

Let STRIPS\(_{1}^{1}\) be the class of planning tasks with actions that have one precondition and one effect. Given two integers \(m\) and \(n\), we use \(\llbracket m,n\rrbracket\) to denote the (possibly empty) set \(\{m,\ldots,n\}\), and we use \(\llbracket n\rrbracket\) to denote the (possibly empty) set \(\{1,\ldots,n\}\). We may drop the parenthesis for singleton sets
(e.g., for preconditions and effects).

\section{Precursor Work on STRIPS\(_{1}^{1}\) }

\citeauthor{strips11workshop} (\citeyear{strips11workshop}) 
came up with initial observations on the STRIPS$_1^1$ problem, 
first exploration results, and its mapping to the induced path length problem in an $n$ dimensional hypercube. 
Among other systems, they proposed schematic ways to generate STRIPS$_1^1$ problems. 
E.g., let ${\cal P}_n$ be the STRIPS$^1_1$ planning task that, for $V = \{ 0,1,2,\ldots \}$, includes the following actions 
\begin{itemize}
  \item for $i \in \llbracket 1,n-1\rrbracket$ take $\langle\operatorname{pre}(i-1), \operatorname{eff}(i)\rangle$ \\ \qquad(add $i$ whenever $i-1$ is true)
  \item for  $i \in \llbracket 1,n-1\rrbracket$ take  $\langle\overline{\operatorname{pre}(i-1)}, \overline{\operatorname{eff}(i)}\rangle$ \\ \qquad(delete $i$ whenever $i-1$ is false)
%\end{itemize}
%and
%\begin{itemize}
  \item 
  $\langle\overline{\operatorname{pre}(n-1)}, \operatorname{eff}(0)\rangle$ (add $0$ when $n-1$ is false)
  \item 
$\langle\overline{\operatorname{pre}(2)}, \operatorname{eff}(0)\rangle$ (add $0$ when $2$ is false)
  \item 
$\langle\operatorname{pre}(n-1), \overline{\operatorname{eff}(0)}\rangle$ (delete $0$ when $n-1$ is true)
\end{itemize} 
with the initial state set to $(0\cdots 0)$. %${\cal P}_n$ generates long minimal plans. 
E.g.,  ${\cal P}_6$ has actions
\medskip

\begin{tabular}{cccccc} 
$a_0$ &=& $\langle \operatorname{pre}(0),\operatorname{eff}(1) \rangle$, &
$a_1$ &=& $\langle \operatorname{pre}(1),\operatorname{eff}(2) \rangle$, \\
$a_2$ &=& $\langle \operatorname{pre}(2),\operatorname{eff}(3) \rangle$, &
$a_3$ &=& $\langle \operatorname{pre}(3),\operatorname{eff}(4)\rangle$, \\
$a_4$ &=& $\langle \operatorname{pre}(4),\operatorname{eff}(5)\rangle$, &
$a_5$ &=& $\langle \overline{\operatorname{pre}(5)},\operatorname{eff}(0)\rangle$, \\
$a_6$ &=& $\langle \overline{\operatorname{pre}(0)},\overline{\operatorname{eff}(1)}\rangle$, &
$a_7$ &=& $\langle \overline{\operatorname{pre}(1)},\overline{\operatorname{eff}(2)}\rangle$, \\
$a_8$ &=& $\langle \overline{\operatorname{pre}(2)},\overline{\operatorname{eff}(3)}\rangle$, &
$a_9$ &=& $\langle \overline{\operatorname{pre}(3)},\overline{\operatorname{eff}(4)}\rangle$, \\
$a_{10}$ &=& $\langle \overline{\operatorname{pre}(4)},\overline{\operatorname{eff}(5)}\rangle$,  &
$a_{11}$ &=& $\langle \operatorname{pre}(5),\overline{\operatorname{eff}(0)}\rangle$, \\
$a_{12}$ &=& $\langle \overline{\operatorname{pre}(2)},\operatorname{eff}(0)\rangle$ 
\end{tabular}

\medskip

generates the following maximal shortest plan (generated backwards from the farthest state in the solver): $101101\leftarrow$ \\ 
%\begin{scriptsize}
$101001\leftarrow101011\leftarrow101010\leftarrow001010\leftarrow011010\leftarrow010010\leftarrow010110\leftarrow010100\leftarrow010101\leftarrow110101\leftarrow100101\leftarrow000101\leftarrow001101\leftarrow011101\leftarrow111101\leftarrow111001\leftarrow110001\leftarrow100001\leftarrow000001\leftarrow000011\leftarrow000111\leftarrow001111\leftarrow011111\leftarrow111111\leftarrow111110\leftarrow111100\leftarrow111000\leftarrow110000\leftarrow100000\leftarrow000000$.
%\end{scriptsize}
%Our program found the goal in depth 30 with 64 states being expanded.
 
The end states of ${\cal P}_n$ were of type $(10)(11)(01)^{k-2}$ for $n=2k$ and states of $(10)(111)(01)^{k-2}$ for $n=2k+1$. The solution lengths grew as follows

\medskip 
\noindent
  \begin{scriptsize}
  \begin{tabular}{c|cc|cc|cc|cc|cc|c}
$k$ & 3 &  3 &  4 &  4 &  5 &  5 &  6 & 6 & 7 & 7 & 8 \\ \hline
$n$ & 6 &  7 &  8 &  9 & 10 & 11 & 12 &  13 &  14 &  15 &  16 \\ \hline
$l$ & 30 & 35 & 49 & 56 & 72 & 81 & 99 & 110 & 130 & 143 & 165  
  \end{tabular}
  \end{scriptsize} 

\medskip

 It was shown that the observed growth of the optimum solution lengths of ${\cal P}_n$ is at most quadratic: for the difference $d(k)=l(k)-l(k-1)$, there are two cases. For $n$ being even we obtain the sequence $d_e(k)= +19,+23,+27,+31,+35,\ldots$ (increase is 4). For $n$ being odd we obtain the sequence $d_o(k)=+21,+25,+29,+34,+38\ldots$ (increase is 4). 

In both cases $d(k) - d(k-1) = 4$, which implies $d(k) = 4 + d(k-1) = ... = 4\cdot k + c$. To determine constant $c$, we distinguish $n$ even, where $d(3) = 19 = 12 + c_e$ implies $c_e = 7$, while $n$ odd, where $d(3) = 21 = 12 + c_o$ implies $c_o = 9$ 
We have $l(k) - l(k-1) = d(k-1)$, so that $l(k) = l(k-1) + d(k-1)$. Given the closed form for $d$ yields
\begin{eqnarray*}
l(k) &=& l(k-1) + (4k+c) \\
     &=& l(k-2) + (4k+c) + (4(k-1)+c) \\
     &=& \ldots = b + ck + 4 (k (0+1+\ldots+k)) 
     % &=& b + ck + 2 (k (k-1)) 
     =\Theta(n^2)
\end{eqnarray*}
Similarly, $l(n)=\Theta(n^2)$ for $k = n/2$. As sharp bounds for $n$ even, i.e., $n=2k$, we have $l(3) = 30 = b_e + 7\cdot 3 + 2 \cdot 6$ yields $b_e = 30 - 21 - 12 = -3$ and $l(k) = -3 + 5k + 2k^2$; and for $n$ odd, i.e.,  $n=2k+1$, we have $l(3) = 35 = b_o + 9\cdot 3 + 2 \cdot 6$ yields $b_o = 35 - 27 - 12 = -4$ and $l(k) = -4 + 7k + 2k^2$. 

The authors introduce two more systems. 
${\cal Q}_n$ and ${\cal R}_n$.
While for small values of $n$ the system ${\cal Q}_n$ has longer shortest solutions than ${\cal P}_n$ and ${\cal R}_n$ 
Table~\ref{tab:growth3length} illustrates that for larger values of $n$ this no longer is the case. While ${\cal R}_n$ strictly outperformed ${\cal P}_n$ we can see that the optimal solution lengths of ${\cal R}_n$ outgrow ${\cal P}_n$ only by a linear amount of at most $(k+2)$. 

\begin{table}[t]
  \centering
  \begin{scriptsize}
  
  \begin{tabular}{c|cc|cc|cc|cc|cc|c}
$k$ & 3 &  3 &  4 &  4 &  5 &  5 &  6 & 6 & 7 & 7 & 8 \\ \hline
$n$ & 6 &  7 &  8 &  9 & 10 & 11 & 12 &  13 &  14 &  15 &  16 \\ \hline
$l_P$ & 30 & 35 & 49 & 56 & 72 & 81 & 99 & 110 & 130 & 143 & 165 \\ \hline 
$l_Q$ & {\bf 38} & {\bf 47} & {\bf 57} & {\bf 65} & 74 & 83 & 92 & 101 & 110 & 119 & 128 \\  \hline 
$l_R$ & {31} & {38} & 52 & 59 & {\bf 77} & {\bf 86} & {\bf 106} & {\bf 117} & {\bf 139} & {\bf 152} & {\bf 176}  
  \end{tabular}
  \end{scriptsize}
  \caption{Growth of optimal solution length $l_p$, $l_q$ and $l_r$ wrt.\ $n$ and $k$ in ${\cal P}_n$, ${\cal Q}_n$ and  ${\cal R}_n$, respectively. }
  \label{tab:growth3length}
\end{table}
 
Although these experiments on scaling systems do not prove that solution lengths of all possible systems grow polynomially, based on further experiments and different systems, the authors conjectured that STRIPS\(_{1}^{1}\) will not induce exponentially long plans.

\section{Hypercubes}

Let us first take a look at the $n$-dimensional hypercube over $\{0,1\}^n$ with planning states being mapped to hypercube nodes and STRIPS$^1_1$ actions with a single effect follow one edge. The hypercube on $n$ variables has $2^n$ nodes and $n 2^{n-1}$ undirected edges, or $n2^n$ directed edges. We have $2n \cdot 2n = 4n^2$ different STRIPS$_1^1$ actions. All connect the same number of nodes. For each pair of adjacent nodes, there are $n$ preconditions that hold, i.e.~each edge corresponds to $n$ actions. If we color the hyperedges using $n$ different colors, each action corresponds to $n^2 2^n / 4 n^2 = 2^{n-2}$ directed colored edges, which is an exponential number, quite substantial for a graph with $2^n$ nodes. This makes it hard to form exponentially long plans with no cycles, since each STRIPS$_1^1$ action covers $2^{n-2}$ edges, or $2^{n-1}$ for precondition-free actions.
We have $2^n$ vertices and $n2^{n-1}$ edges in the hypercube $Q_n$.
Every single STRIPS$_1^1$ action fixes two bits: one on the source vector
and one in the target vector, and the one in the target vector must flip, 
so its base, corresponds to a hypercube $Q_{n-2}$ with 
$2^{n-2}$ vertices (and $(n-2)2^{(n-2)-1} = (n-2) 2^{n-3}$ edges). 
The number of the $m$-dimensional hypercubes (alias $m$-cubes) 
contained in the boundary of an $n$-cube is
$2^{n-m}{n \choose m}$
with $m=n-2$ we have 
$$2^{n-(n-2)}{n \choose n-2}=\frac{4(n!)}{2 (n-2)!} = 2n(n-1)$$
different cubes $Q_{n-2}$ in $Q_n$.
    
With $n$ variables, positive or negative in preconditions and effects, there are at most $(2n)(2n) = 4n^2 = \Theta(n^2)$ 
different STRIPS$_1^1$ actions. As we do not allow repeating variables, then this value
reduces to $2n(2(n-1)) = 4n^2-4n= \Theta(n^2)$. This is exactly twice the number of hypercubes
$Q_{n-2}$ as preconditions and effects can be exchanged without changing the cube. Otherwise
each STRIPS$_1^1$ action selects between two orientations of $Q_{n-2}$ (see Fig.~\ref{fig:hyper5})

As the problem is symmetric and the variables are interchangeable, without loss of
generality, we can assume the travel of the longest shortest path 
to start at $(0,\ldots,0)$, so that the problem to find 
the STRIPS$_1^1$ actions generating the longest travel distance 
in fact equivalent to finding the eccentricity of the graph starting at $(0,\ldots,0)$.

\begin{figure}[t]
    \centering
    \includegraphics[width=0.95\linewidth]{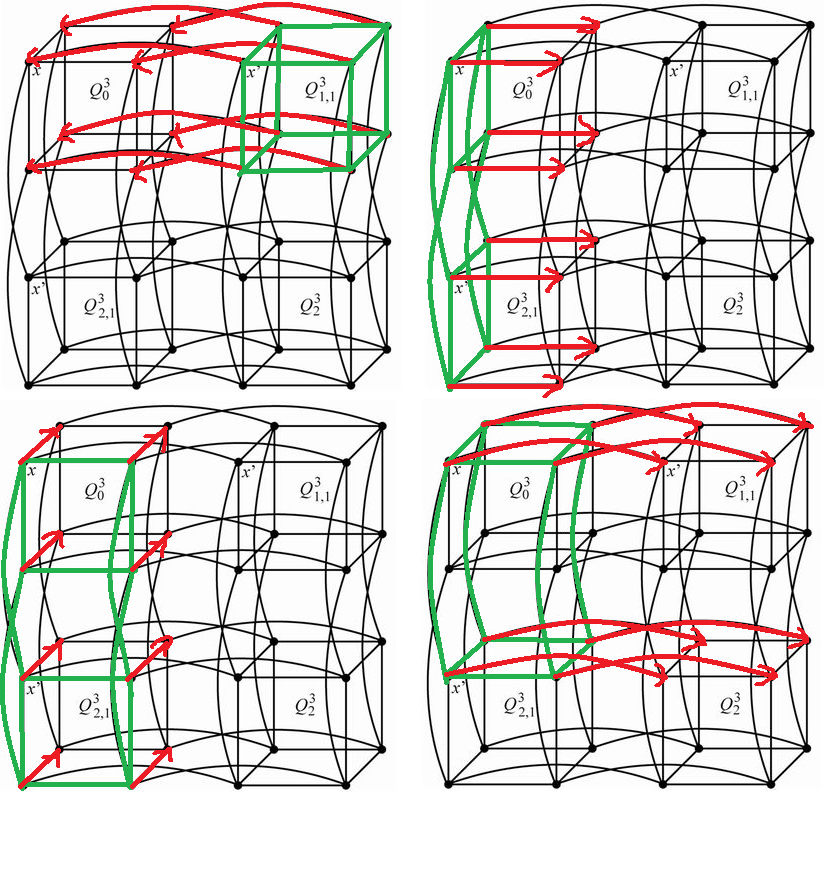}
    \caption{Sketch of four selected STRIPS$^1_1$ actions in the $5$-dimensional hypercube $Q_5$, green lines indicating the start and red arrows indicating the end of the actions, there are $2^{n-2} = 8$ edges associated edges with each action and there are overlaps. In other words, green edges correspond to fulfilled preconditions and red edges to action effects.}
    \label{fig:hyper5}
\end{figure}

There is work on recursively constructed directed hypercubes~\cite{DBLP:journals/combinatorics/Domshlak02}. It is shown by construction that
these graphs have one source and one sink and while the diameter is exponential in $n$, the distance
from source to any node is at most $n$ and from any node to the sink,
too. Unfortunately, the direction in the hypercube defined by STRIPS$_1^1$ actions
is not recursively defined, but by directed hypercubes of dimension $n-2$.

Everett and Gupta~(\citeyear{EVERETT1989243}) showed 
that there are directed hypercubes that are acyclic and that
have exponential path length, with a construction that is 
based on Fibonacci numbers. The main observation is that edges
of any (chordless) path can be oriented so that this path is shortest.

Gregor~(\citeyear{PetrGregor}) studied hypercubes in his PhD thesis including the
findings of Hamiltonian cycles and the inclusion of faulty edges to be 
avoided. He also looks at Fibonacci cubes as subcubes of
the hypercube, where the only labels that are allowed 
are bitstrings with no two consecutive 1 bits. 

The problem we are facing is related to the snake-in-the-box problem; a longest path problem in a $n$-dimensional hypercube. The design of a long snake has impact for the generation of improved error-correcting codes. The snake increases in length, but must not approach any of its previous visited vertices with Hamming distance 1 or less. It has been proven that the snake-in-the-box problem generates longest plans that are exponentially large~\cite{DBLP:journals/jct/AbbottK88,DBLP:journals/dm/AbbottK91}. The formal definition of the problem and its variants as well as heuristic search techniques for solving it have been studied by~\cite{DBLP:conf/socs/PalomboSPFKR15,DBLP:conf/ki/EdelkampC16}.

\section{SAT Encodings}

While results of \citeauthor{strips11workshop} (\citeyear{strips11workshop}) were based on explicit search, we used SAT solvers to find the length of a longest plan for a given $n$. We used \emph{kissat} 
%(for keep it simple and clean bare metal SAT solver) 
developed by \citeauthor{Biere} (\citeyear{Biere}) as a single thread SAT solver.
It is described as the port of another solver (CaDiCaL) solver back to C with improved data structures, better scheduling of inprocessing and optimized algorithms and implementation.

The conjunctive normal form (CNF) we generated has three groups of variables:
action variables decide for every two literals whether there is an action, path variables determine the state for every time, and breadth-first search (BFS) variables determine whether a state is reachable at a time. Note that BFS variables only give an implication: If a state $u$ is reachable in time $t$, then the corresponding variable is true. Therefore, path and BFS variables are joined together: if $u$ is $t$-th state on the path, then BFS variable for $u$ in time $t-1$ must be false.
In other words, path variables prove existence of a solution of a given length, and BFS variables prove non-existence of a shorter path.

For 5 variables, we obtained a computer-assisted proof that 30 states (29 actions) are optimal. Therefore, the maximal possible path length is strictly below $2^n$ and not all states can be reached.
For 6 variables, we found a solution plan with 40 states (39 actions), which is 1 more than the best result reported by \citeauthor{strips11workshop} (\citeyear{strips11workshop}). We verified that there is no longer shortest plan for $n=6$. The computation took about 2 weeks on one core of a contemporary CPU. For 7 variables, we found a plan with 51 states (50 actions)
%(see Fig.~\ref{fig:sat751}), 
but could not prove a matching upper bound.
%\begin{figure}
%    \centering
%\includegraphics[width=0.49\linewidth]
%{AnonymousSubmission/LaTeX/bfs_7_51.pdf}
%    \caption{Best SAT solver result for $n=7$.}
%    \label{fig:sat751}
%\end{figure}

\section{Shortest Plans in Good Tasks}
%omit $\ge 25\%$ States}

Given a STRIPS\(^1_1\) planning instance $\mathbf{p}=\langle V, A, s_I, s_G\rangle$ we may assume for simplicity that there is no action $a\in A$ with $\pre(a)=\overline{\eff(a)}$. Such an action can be replaced with two actions $(u,\eff(a))$, $(\overline{u},\eff(a))$ for some other variable $u$.

Furthermore, we may assume without loss of generality that for every literal $l\in L(V)$ there exists an action $a\in A$ with $\eff(a)=l$; for otherwise, in any plan for $\mathbf{p}$ there is at most one action affecting that variable. In this section we consider a mildly stronger assumption as follows.

\begin{definition}\label{def:good} 
We say that an instance $\mathbf{p}=\langle V, A, s_I, s_G\rangle$ is \emph{good} if for every variable $v \in V$ there exist two actions $a,b \in A$ with $\eff(a)=v$, $\eff(b)=\overline{v}$ and $\pre(a)\ne \overline{\pre(b)}$.
\end{definition}

Note that for example every instance with at least three actions affecting each variable is good. If an instance is not good, for some variable $v$ there are only two \emph{parallel} actions affecting $v$, namely $(l,v)$ and $(\overline{l}, \overline{v})$ for some literal $l$. %(bad instances are not treated here, but it seems that since effects on $v$ require effects on $l$, these variables could be ``glued'').

Let $Q_n(\mathbf{p})$ denote the directed graph on all $2^n$ possible states as vertices, and edges from states to all possible consecutive states. It is a spanning subgraph of the $n$-dimensional bidirected hypercube. Each variable $v\in V$ corresponds to one coordinate of the hypercube. Each action $a=(l_1,l_2)\in A$ induces directed edges $\{(s,s\oplus{l_2}) \mid  l_1,\overline{l_2}\in s\}\subseteq E(Q_n(\mathbf{p}))$ .  If we have edges between $s_1$ and $s_2$ in both directions, i.e. $(s_1,s_2),(s_2,s_1)\in E(Q_n(\mathbf{p}))$, we say that $\{s_1,s_2\}$ is a \emph{bidirectional} edge. Note that bidirectional edges are induced by two or more actions with effects on opposite literals of the same variable.

\begin{lemma}
\label{obs:bidir}
  Let $\mathbf{p}$ be a good instance on $n$ variables. For every variable $v$ the graph $Q_n(\mathbf{p})$ contains at least $2^{n-3}$ bidirectional edges in the coordinate $v$.
\end{lemma}

\noindent
{\bf Proof.}
  Let $(l_1,v)$ and $(l_2,\overline{v})$ be two actions with $l_1 \ne \overline{l_2}$. Since the literals $l_1$ and $l_2$ are consistent, the set $\{\{s,s\oplus \overline{v}\} \mid l_1,l_2,v\in s\}$ contains $2^{n-3}$ if $l_1 \ne l_2$, and $2^{n-2}$ if $l_1=l_2$, bidirectional edges in the coordinate $v$.\hfill $\Box$

\begin{lemma}
\label{cor:bidir} 
  Let $\mathbf{p}$ be a good instance on $n$ variables. Then $Q_n(\mathbf{p})$ contains at least $2^{n-2}$ vertices such that each of them is incident with at least $n/4$ bidirectional edges. %\footnote{
  %I rather chose to keep it simple.
\end{lemma}

\noindent
{\bf Proof.}
  There are altogether at least $n2^{n-3}$ bidirectional edges, which means at least $n2^{n-2}$ incidences and average degree at least $n/4$. On the other hand, since each vertex has at most $n$ incidences, they are distributed on at least $2^{n-2}$ vertices with at least average degree. \hfill $\Box$

\medskip

The number $2^{n-2}$ can be perhaps slightly improved since if all incidences are on this number of vertices, they form a subcube of dimension $n-2$, but, then, the degree is only $n-2$ not $n$.

\begin{theorem}
\label{thm:omit}
  Let $\mathbf{p}$ be a good instance on $n$ variables and let $\omega$ be a shortest plan for $\mathbf{p}$. If $n \ge 16$, then $\omega$ omits at least $2^{n-2}$ states, i.e. it has length at most $3\cdot2^{n-2}$. If $12\le n <16$, then $\omega$ omits at least $2^{n-3}$ states.
\end{theorem}

\noindent
{\bf Proof.}
  Let $S$ be the set of vertices (states) in $Q_n(\mathbf{p})$ that are incident with at least $n/4$ bidirectional edges. We have that $|S|\ge 2^{n-2}$ by Lemma~\ref{cor:bidir}. For $s \in S$ let $N_{b}(s)$ denote the set of neighbors of $s$ over the bidirectional edges.
  
  We observe that if $s \in S$ is visited by $\omega$, then $\omega$ visits at most two vertices from $N_b(S)$, namely the predecessor and the successor of $s$ on $\omega$ can both be from $N_b(S)$. Other vertices from $N_b(s)$ cannot be visited s before nor after $s$ since otherwise the bidirectional edge would create a shortcut on $\omega$.
  Furthermore, if $s_1,s_2 \in S$ are both visited by $\omega$, say $s_1$ before $s_2$, and $N_b(s_1)$ and $N_b(s_2)$ intersect in two vertices, then one of them is also visited by $\omega$, namely as a successor of $s_1$ and a predecessor of $s_2$. Therefore, if $k$ denotes the number of vertices of $S$ omitted by $\omega$, the total number of omitted vertices is at least
  $$k+(2^{n-2}-k)(n/4-3)\ge 2^{n-2} \quad \text{if }n\ge 16.$$
   
  For $12\le n<16$, if $s \in S$ is visited by $\omega$, $|N_b(S)|=3$, and both the predecessor and the successor on $\omega$ are from $N_b(S)$, then it can share the remaining neighbor from $N_b(S)$ only with one vertex at distance two on $\omega$. So these shared neighbors are counted only in pairs and we obtain at least
  $$k+(2^{n-2}-k)(n/4-2)-(2^{n-2}-k)/2$$
  omitted vertices, which has its minimum $2^{n-3}$ at $k=0$.
\hfill $\Box$

\section{Literal Graph}

Given a STRIPS\(_{1}^{1}\) planning task \(\langle V,A\rangle\), a literal graph is a directed graph \(G=\langle L(V),A\rangle\) with literals in \(L(V)\) as nodes and actions in \(A\) as edges. Note that in STRIPS\(_{1}^{1}\), we can write each action \(a\in A\) on the form \(a=\langle\mathrm{pre}(a),\mathrm{eff}(a)\rangle=\langle l_{1},l_{2}\rangle\), where \(l_{1}\) and \(l_{2}\) are literals in \(L(V)\), i.e. \(\mathrm{pre}(a)=l_{1}\) and \(\mathrm{eff}(a)=l_{2}\) are literals instead of literal sets and \(a\) can be interpreted as a directed edge between \(l_{1}\) and \(l_{2}\) in \(G\).

Consider a planning instance \(\bm{p}=\langle V,A,s_{I},s_{G}\rangle\) and a plan \(\omega=\langle a_{1},\ldots,a_{k}\rangle\). For each \(l\in L(V)\) and \(i\in[[0,k]]\), we recursively construct a subplan \(\omega_{l,i}\) of \(\omega\) as follows as 

\noindent
$\omega_{l,i}=$
\[
\left\{\begin{array}{ll}
\langle\omega_{\mathrm{pre}(a_{j}),j-1},a_{j}\rangle & \exists j\in\llbracket i\rrbracket\textrm{ s.t. }\mathrm{eff}(a_{j})=l\ \wedge  \\
& \ \ \mathrm{eff}(a_{\ell})\neq l,\forall\ell\in\llbracket j+1,i\rrbracket,\\
\langle\rangle & l\in s_{I}\ \wedge\, \not{\exists}j\in\llbracket i\rrbracket\textrm{ s.t. }\mathrm{eff}(a_{j})=l,\\
\bot & l\notin s_{I}\ \wedge\,\not{\exists} j\in\llbracket i\rrbracket\textrm{ s.t. }\mathrm{eff}(a_{j})=l.
\end{array}\right.
\]

Here, \(\langle\omega_{\mathrm{pre}(a_{j}),j-1},a_{j}\rangle\) denotes the subplan that results from appending action \(a_{j}\) to the subplan \(\omega_{\mathrm{pre}(a_{j}),j-1}\), \(\langle\,\rangle\) denotes the empty plan, and \(\bot\) denotes an illegal plan. Intuitively, if there exists an action in \(a_{1},\ldots,a_{i}\) that causes \(l\) to be true, the subplan \(\omega_{l,i}\) ends with the last such action \(a_{j}\), and begins with the subplan \(\omega_{\mathrm{pre}(a_{j}),j-1}\) that makes the precondition of \(a_{j}\) true prior to time step \(j\). If no such action exists, the subplan \(\omega_{l,i}\) is empty if \(l\in s_{I}\), and otherwise illegal since the plan \(\omega\) does not make \(l\) true by time step \(i\). It is simple to compute each \(\omega_{l,i}\) using dynamic programming.

%\section{Examples and Overlap Trees}

In the following, we show several examples of literal graphs and overlap trees.
with a number of variables in \([[4,7]]\).
On the edges, each example provides the indices of the actions in the obtained (shortest) plan and the subplans that end in a goal literal are colored. 

It is relatively easy to see that subplans can have common prefixes, but once they diverge, they remain separate throughout the solution plan. Hence, we organize the subplans in an \emph{overlap tree} structure. There are some interesting observations we can make from the examples. Once two subplans diverge, eventually the subplans operate on distinct sets of actions (edges), which are colored in the literal graph and in the overlap trees. There may be a transition period during which a subplan selects actions with the goal of reaching its particular subset of actions. A subplan may cycle multiple times over its subset of actions. 
%
%Even though a subplan may cycle multiple times over its subset of actions, the number of such cycles is limited by the number of other subplans and the number of actions in the subset. 
%
%If we were to prove that the subset of actions of subplans is always different, this could imply a polynomial bound on the total number of actions.

An example for \(n=4\) with plan length 15 is as follows.

\includegraphics[width=7.5cm]{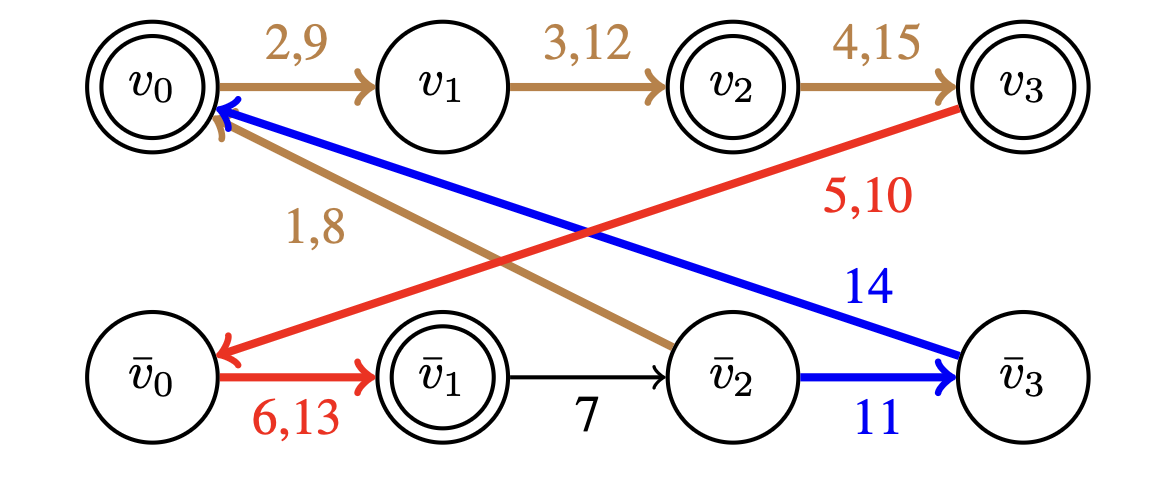}

\includegraphics[width=8.5cm]{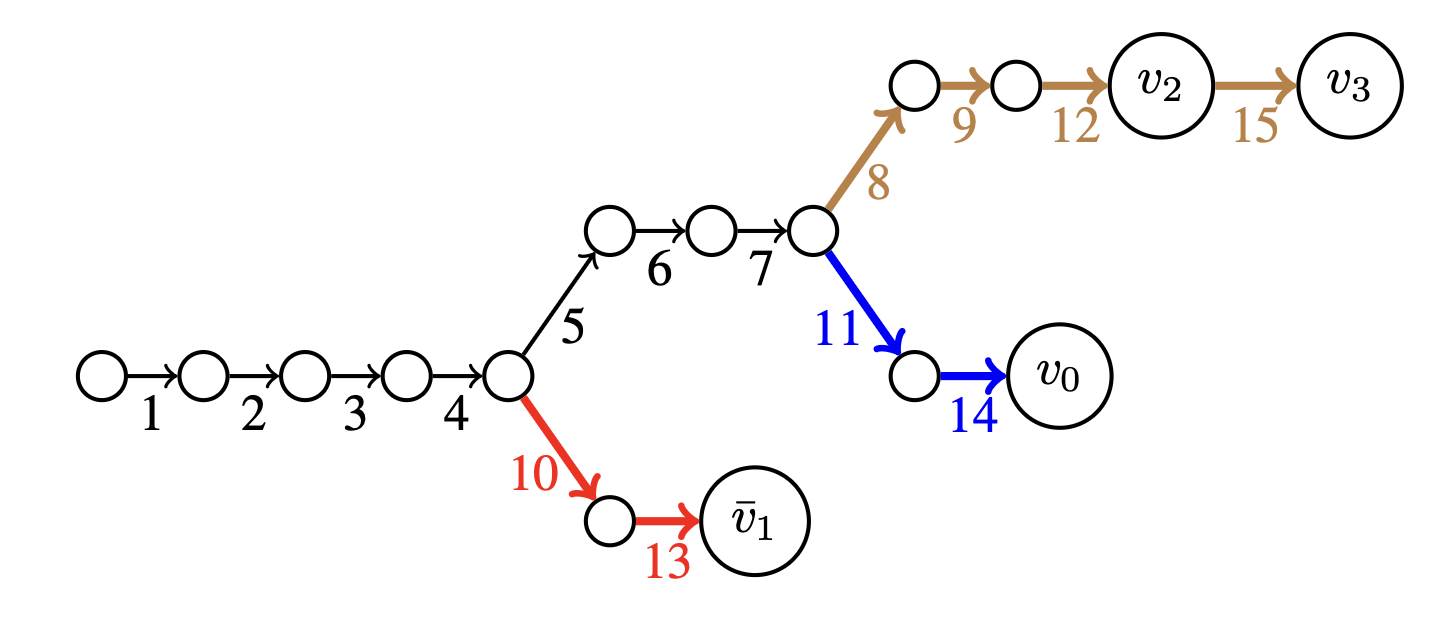}
%\begin{itemize}
%    \item Path for \(v_{0}\): 1,2,3,4,5,6,7,\textbf{11,14}
%    \item Path for \(\bar{v}_{1}\): 1,2,3,4,\textbf{10,13}
%    \item Path for \(v_{3}\): 1,2,3,4,5,6,7,8,9,12,\textbf{15}
%\end{itemize}

%\pagebreak

{Example for \(n=5\) with plan length 29}

\includegraphics[width=7.5cm]{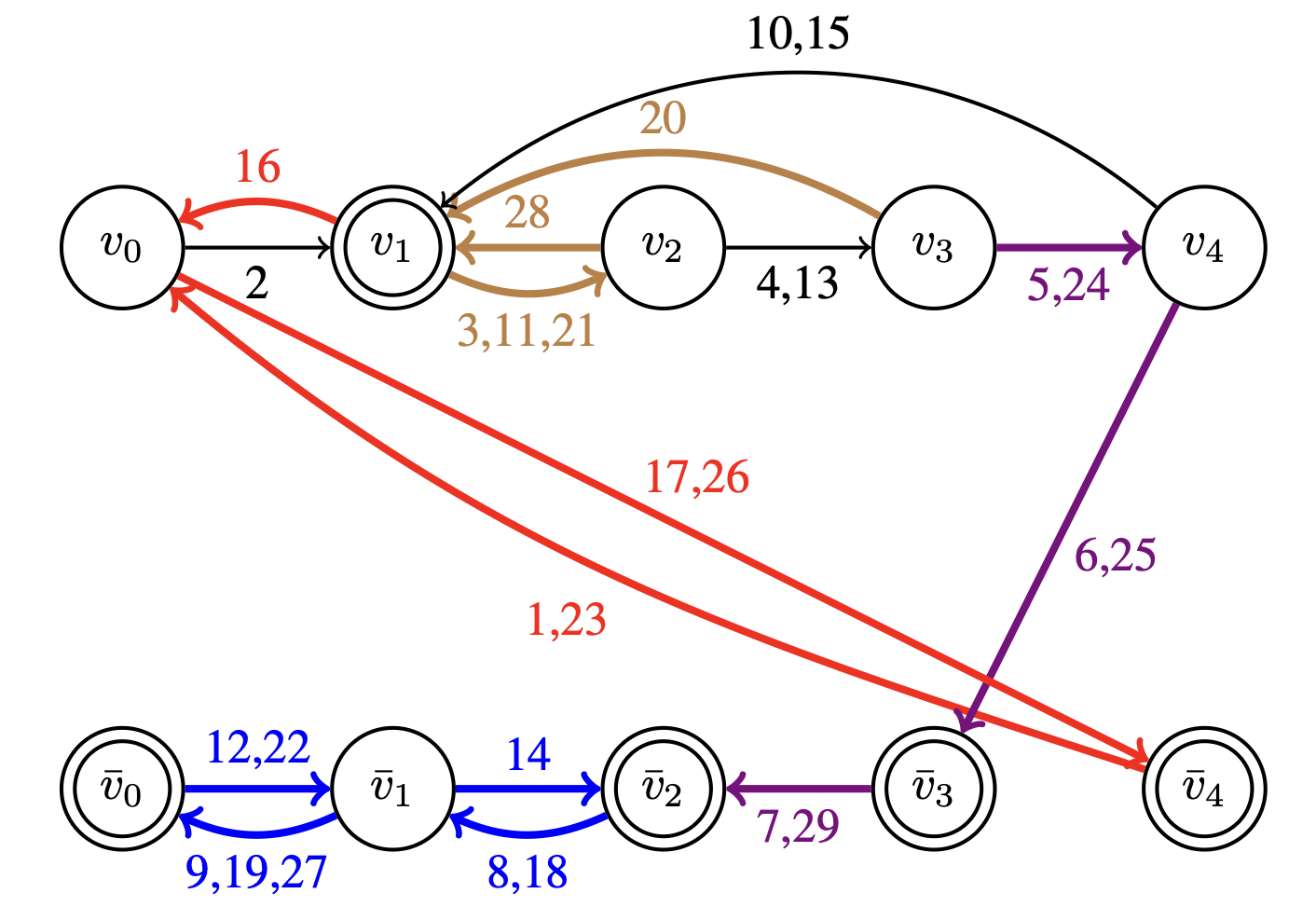}
\includegraphics[width=8.5cm]{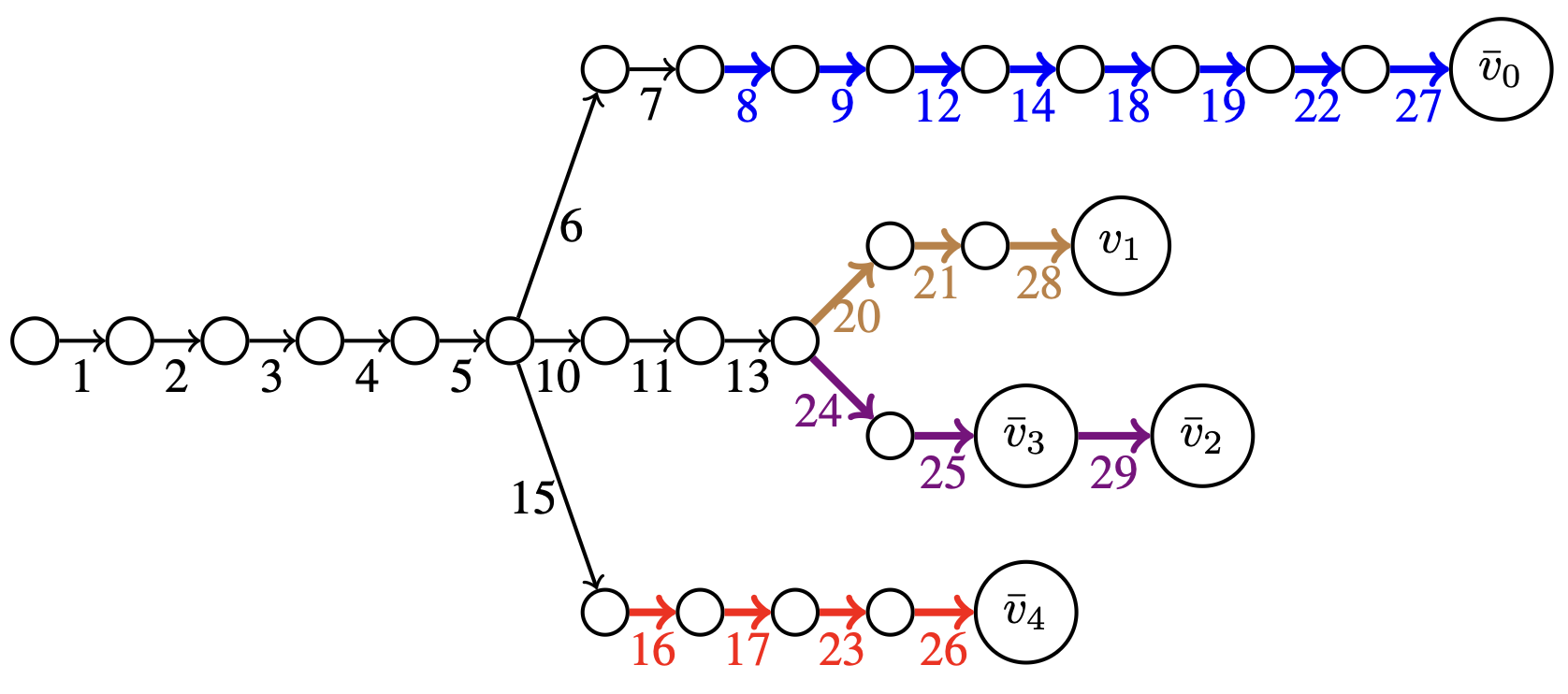}

%\begin{itemize}
%    \item Path for \(\bar{v}_0\): 1,2,3,4,5,6,7,8,9,12,14,18,19,22,27
%    \item Path for \(v_1\): 1,2,3,4,5,10,11,13,20,21,28
%    \item Path for \(\bar{v}_2\): 1,2,3,4,5,10,11,13,24,25,29
%    \item Path for \(\bar{v}_3\): 1,2,3,4,5,10,11,13,24,25
%    \item Path for \(\bar{v}_4\): 1,2,3,4,5,15,16,17,23,26
%\end{itemize}

%\pagebreak
 
{Example for \(n=6\) with plan length 38}

\includegraphics[width=8.5cm]{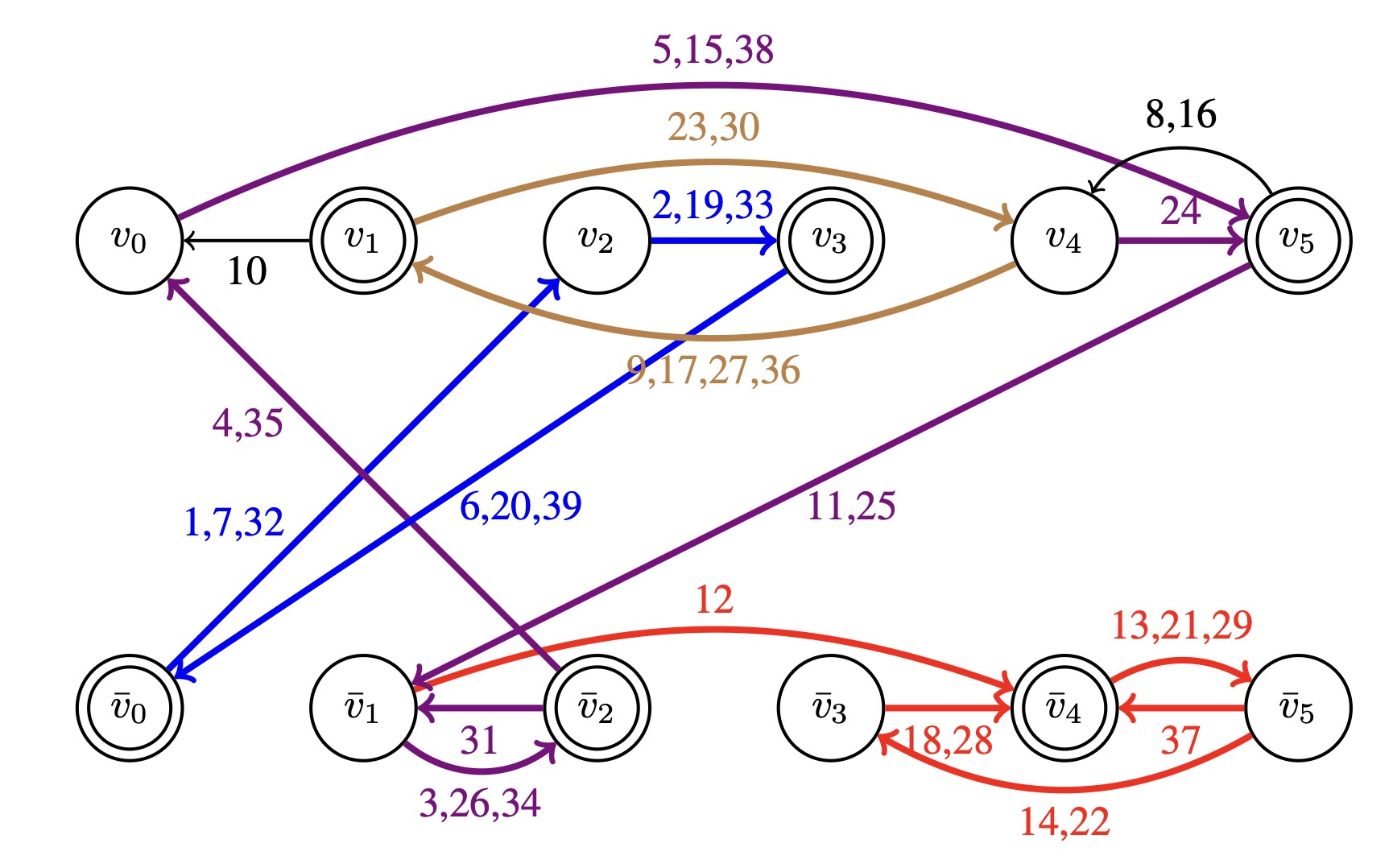} 

%\pagebreak 

\includegraphics[width=8.5cm]{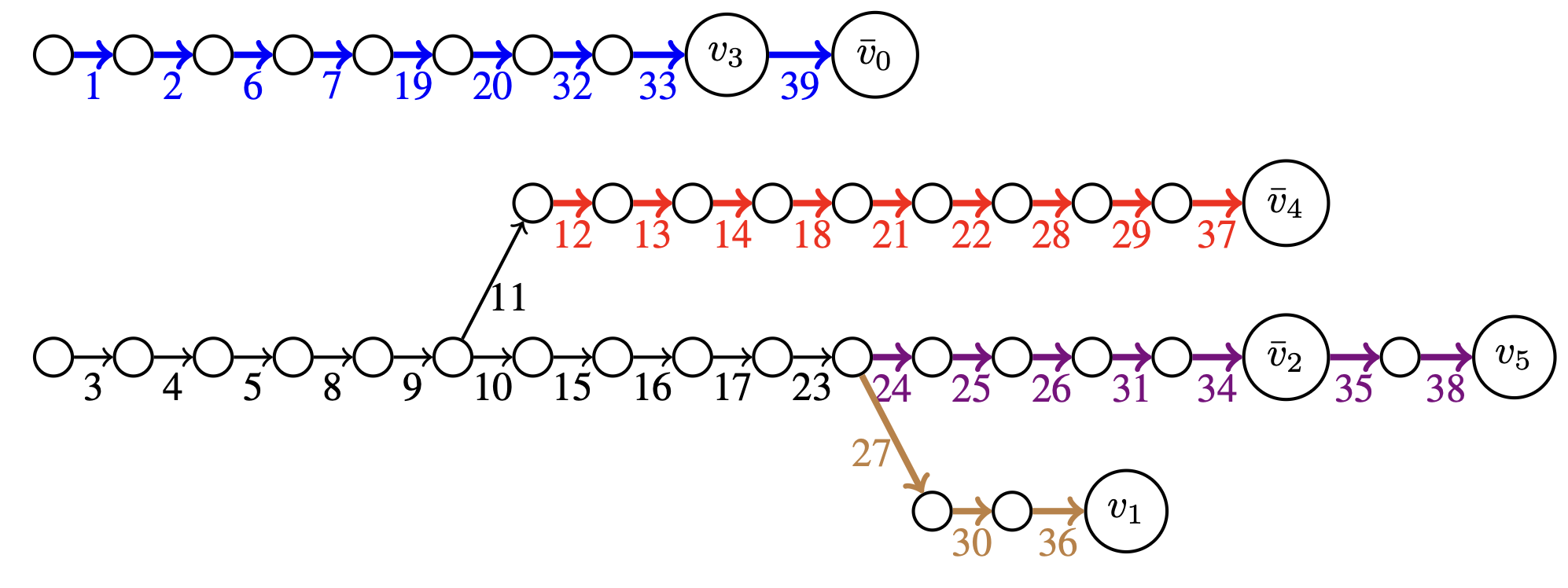}

%\begin{itemize}
%    \item Path for \(\bar{v}_0\): 1,2,6,7,19,20,32,33,39
%    \item Path for \(v_1\): 3,4,5,8,9,10,15,16,17,23,27,30,36
%    \item Path for \(\bar{v}_2\): 3,4,5,8,9,10,15,16,17,23,24,25,26,31,34
%    \item Path for \(v_3\): 1,2,6,7,19,20,32,33
%    \item Path for \(\bar{v}_4\): 3,4,5,8,9,10,11,12,13,14,18,21,22,28,29,37
%    \item Path for \(v_5\): 3,4,5,8,9,10,15,16,17,23,24,25,26,31,34,35,38
%\end{itemize}

{Example for \(n=7\) with plan length 50}

\includegraphics[width=8.5cm]{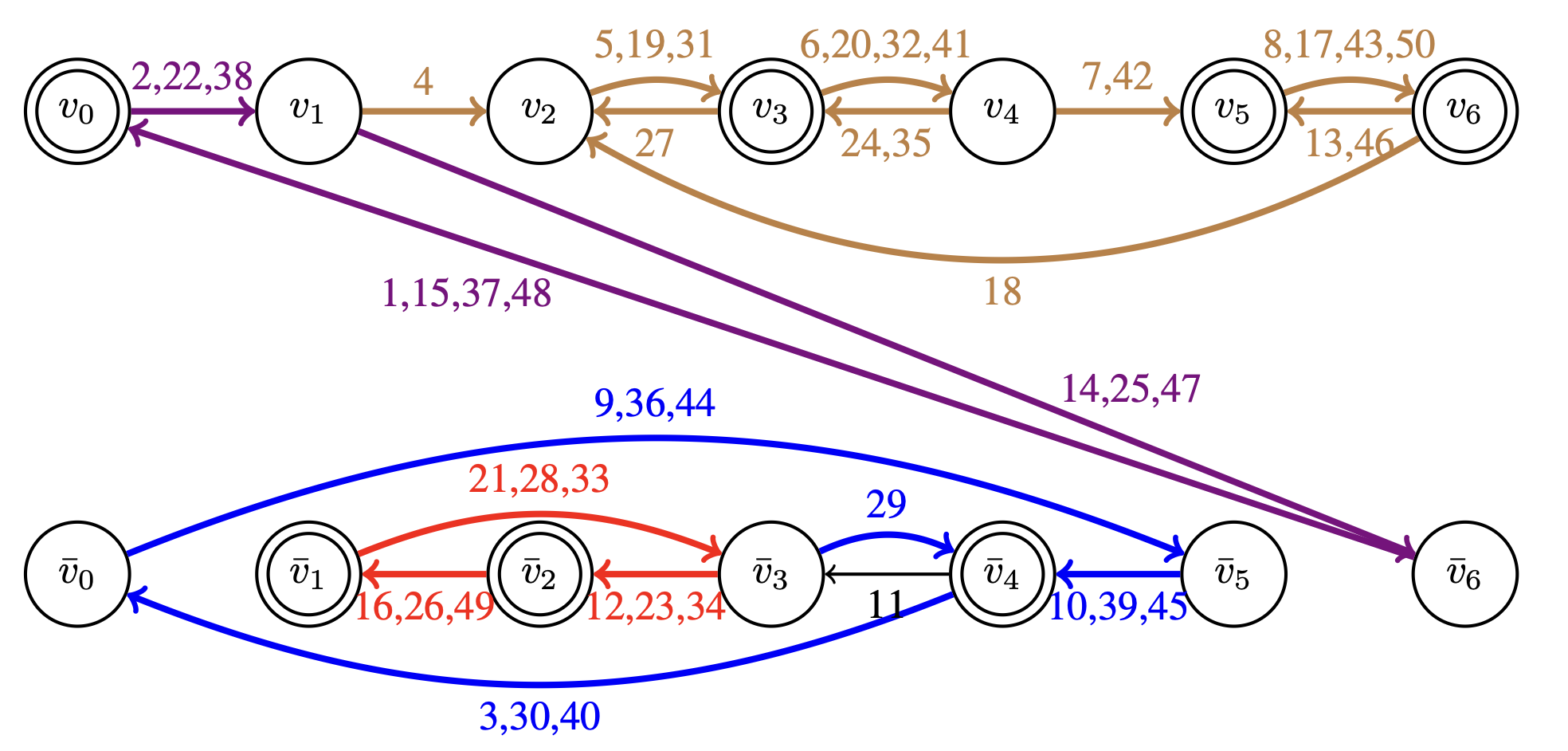}
\includegraphics[width=9cm]{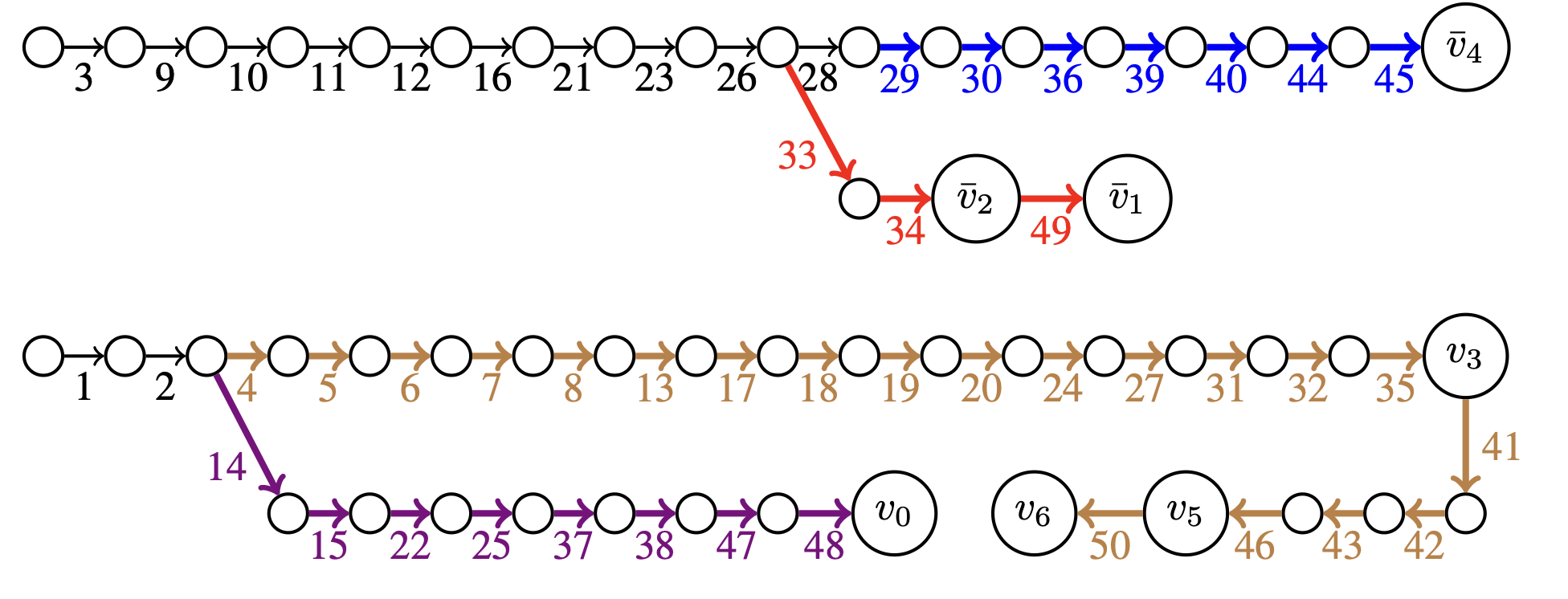}

%\begin{itemize}
%    \item Path for \(v_0\): 1,2,14,15,22,25,37,38,47,48
%    \item Path for \(\bar{v}_1\): 3,9,10,11,12,16,21,23,26,33,34,49
%    \item Path for \(\bar{v}_2\): 3,9,10,11,12,16,21,23,26,33,34
%    \item Path for \(v_3\): 1,2,4,5,6,7,8,13,17,18,19,20,24,27,31,32,35
%    \item Path for \(\bar{v}_4\): 3,9,10,11,12,16,21,23,26,28,29,30,36,39,40,44,45
%    \item Path for \(v_5\): 1,2,4,5,6,7,8,13,17,18,19,20,24,27,31,32,35,41,42,43,46
%    \item Path for \(v_6\): 1,2,4,5,6,7,8,13,17,18,19,20,24,27,31,32,35,41,42,43,46,50
%\end{itemize}
 
\pagebreak

\begin{figure}[t]
    \centering
\begin{tikzpicture}[scale=0.4,
    % Define styles for the nodes
    circnode/.style={circle, draw, fill=blue!20, minimum size=8mm},
    dcircnode/.style={double, circle, draw, fill=blue!20, minimum size=8mm},
    sqnode/.style={rectangle, draw, fill=green!20, minimum size=1mm},
    edge/.style = {line width = 1.3pt,->,> = latex'} 
]
\begin{scope}[node distance=20mm, yshift=0cm]
    % First set of 4 circular nodes
    \node[circnode, "$S^1$" label position=left] (c1a) at (0, 14) {P1};
    \node[dcircnode,"$G^2$","$S^2$" label position=left] (c2a) at (0, 10) {P2};
    \node[circnode, "$S^3$" label position=left] (c3a) at (0, 6) {P3};
    \node[circnode, "$S^4$" label position=left](c4a) at (0, 2) {P4};
\end{scope}

\begin{scope}[node distance=20mm, xshift=16cm]
    % Second set of 4 circular nodes
    \node[dcircnode,"$G^1$" label position=right] (c1b) at (0, 14) {P5};
    \node[circnode] (c2b) at (0, 10) {P6};
    \node[dcircnode,"$G^3$" label position=right] (c3b) at (0, 6) {P7};
    \node[dcircnode,"$G^4$" label position=right] (c4b) at (0, 2) {P8};
\end{scope}

% Define a scope for the second partition (square nodes)
% We'll place 9 square nodes on the right side

\begin{scope}[node distance=20mm, xshift=8cm, yshift=8cm]
    \node[sqnode] (s1) at (0, 8) {$T1$};
    \node[sqnode] (s2) at (0, 6) {$T2$};
    \node[sqnode] (s3) at (0, 4) {$T3$};
    \node[sqnode] (s4) at (0, 2) {$T4$};
    \node[sqnode] (s5) at (0, 0) {$T5$};
    \node[sqnode] (s6) at (0, -2) {$T6$};
    \node[sqnode] (s7) at (0, -4) {$T7$};
    \node[sqnode] (s8) at (0, -6) {$T8$};
    \node[sqnode] (s9) at (0, -8) {$T9$};
\end{scope}

% Draw edges between nodes from different partitions
% (circular nodes connect only to square nodes)

% Connections from the first group of circular nodes

% --- undir edges

\draw[edge,blue] (s1) to (c1a); 
\draw[edge,blue] (c1b) to (s1); 

\draw[edge,blue] (c1a) to (s2);
\draw[edge,blue] (s2) to (c1b);

\draw[edge,blue] (c1a) to (s3);
\draw[edge,blue] (s3) to (c1b);

\draw[edge,blue] (c2a) to (s4);
\draw[edge,blue] (s4) to (c2b);

\draw[edge,blue] (c2b) to (s5);
\draw[edge,blue] (s5) to (c2a);

\draw[edge,blue] (c3a) to (s6);
\draw[edge,blue] (s6) to (c3b);
\draw[edge,blue] (c3b) to (s7);
\draw[edge,blue] (s7) to (c3a);

\draw[edge,blue] (c4a) to (s8);
\draw[edge,blue] (s8) to (c4b);
\draw[edge,blue] (c4b) to (s9);
\draw[edge,blue] (s9) to (c4a);

% ------ bidir edges

%\draw[edge,red] (c4a) to (s1);
%\draw[edge,red] (s1) to (c4a);
%\draw[edge,red] (c4b) to (s2);
%\draw[edge,red] (s2) to (c4b);
%\draw[edge,red] (c3b) to (s3);
%\draw[edge,red] (s3) to (c3b);

%\draw[edge,red] (c1a) to (s4);
%\draw[edge,red] (s4) to (c1a);
%\draw[edge,red] (c1b) to (s5);
%\draw[edge,red] (s5) to (c1b);

%\draw[edge,red] (c2a) to (s6);
%\draw[edge,red] (s6) to (c2a);
%\draw[edge,red] (c2b) to (s7);
%\draw[edge,red] (s7) to (c2b);

%\draw[edge,red] (c3b) to (s9);
%\draw[edge,red] (s9) to (c3b);
%\draw[edge,red] (c3a) to (s8);
%\draw[edge,red] (s8) to (c3a);
% ------ bidir edges

\draw[edge,red] (c4b) to (s1);
\draw[edge,red] (s1) to (c4b);
\draw[edge,red] (c4a) to (s2);
\draw[edge,red] (s2) to (c4a);
\draw[edge,red] (c3a) to (s3);
\draw[edge,red] (s3) to (c3a);

\draw[edge,red] (c1b) to (s4);
\draw[edge,red] (s4) to (c1b);
\draw[edge,red] (c1a) to (s5);
\draw[edge,red] (s5) to (c1a);

\draw[edge,red] (c2b) to (s6);
\draw[edge,red] (s6) to (c2b);
\draw[edge,red] (c2a) to (s7);
\draw[edge,red] (s7) to (c2a);

\draw[edge,red] (c3a) to (s9);
\draw[edge,red] (s9) to (c3a);
\draw[edge,red] (c3b) to (s8);
\draw[edge,red] (s8) to (c3b);

%\draw[edge] (c4a) to (s8);

% Connections from the second group of circular nodes
%\draw (c1b) to (s1);
%\draw (c1b) to (s3);
%\draw (c2b) to (s2);
%\draw (c2b) to (s4);
%\draw (c3b) to (s5);
%\draw (c3b) to (s7);
%\draw (c4b) to (s6);
%\draw (c4b) to (s9);
\end{tikzpicture}
         \caption{
         Result of compilation scheme for the above $(n=4)$ STRIPS$^1_1$ example into safe conservative Petri-Net and into a cooperative MAPF, the literals represent nodes that correspond to the literal graph, they are negated on the left-hand side nad positive on the right-hand side. }
    \label{fig:mapping}
\end{figure}
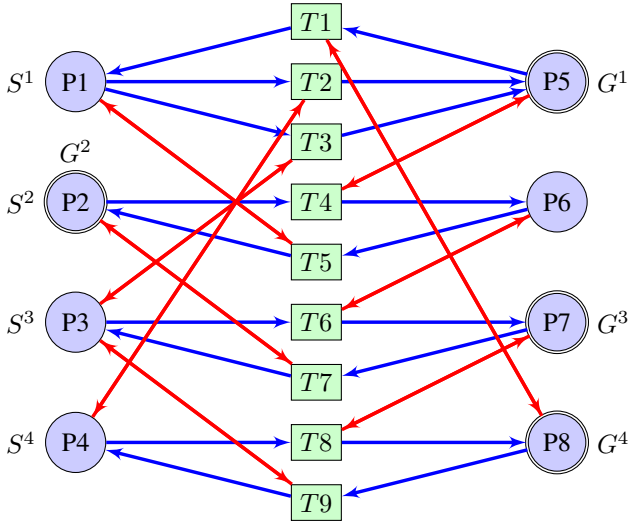 

%$a = (\mbox{\em pre},\mbox{\em eff}) \in$ STRIPS$^1_1$ with literals $\mbox{\em pre}$ and $\mbox{\em eff}$ leads to edges in graph 

%$\neg \mbox{\em eff} \rightarrow a \rightarrow \mbox{\em eff}$ and  

%$\neg \mbox{\em pre} \rightarrow a \rightarrow \neg \mbox{\em pre}$ 

%\noindent{\bf Thm} If there is STRIPS$^1_1$ solution, so there is dMAPF solution \\ {\bf Proof idea:} Gate agent goes to $\neg \mbox{\em pre}$, Agent passes by to $\mbox{\em eff}$, Gate agent back.

\section{Compilation Schemes}

Next, we delve into some possible consequences of a complexity result by looking at problem transformations alias compiling schemes.

A general mapping for STRIPS planning problems to 1-safe Petri nets was established and used for task planning via Petri net unfolding
by \citeauthor{petrinetunfolding} (\citeyear{petrinetunfolding}).

Formally, \emph{Petri net} $C = \langle P,T,I,O\rangle$ consists of:
a finite set $P$ of places, a finite set $T$ of transitions, an input function $I: T \rightarrow  2^P$ (maps to bags of places) and an output function $O: T \rightarrow  2^P$.

A \emph{marking} of $C$ is a mapping $m: P\rightarrow \mathbb{N}$. A net $C$ with initial marking $m$ is \emph{safe} if each place always holds at most 1 token. A net is \emph{conservative} if the total number of tokens remains constant. 

Interpreting markings as states, and imposing an initial marking $m_I$ and goal marking $m_G$, there is a mapping $\phi$ of STRIPS\(_{1}^{1}\) planning tasks $p = \langle V,A,s_I,s_G \rangle$ to a conservative 1-safe Petri nets $C = \langle P,T,I,O\rangle$, 
with places $P = \phi(L(V))$, transitions $T = \phi(A)$, and input function $I$ defined by the relations $\phi(a) \rightarrow 
\{
\phi(\operatorname{pre}(a)),
\phi(\overline{\operatorname{eff}(a)}) 
\}$
and output function $O$ defined by the relations $\phi(a) \rightarrow
\{
\phi(\operatorname{pre}(a))
\phi(\operatorname{eff}(a))
\}$
for each 
$a = \langle\operatorname{pre}(a),\operatorname{eff}(a)\rangle \in A$. This mapping fulfills the following property. 

STRIPS\(_{1}^{1}\) is reachable in $t$ steps from $s_I$ to $s_G$
if and only if the marking $m_G$ is reachable in $t$ steps
from marking $m_I$. Any bounded Petri net with a finite reachability set can be modeled by a finite state machine with one state for each reachable marking, indicating which transitions fire between the states. Although reachability in safe Petri nets is PSPACE-complete~\cite{CHENG1995117}, we do not yet know whether being additionally conservative casts it as NP-complete. The usual Petri net unfolding algorithms do not directly provide a polynomial time bound.

Moreover, there is also a possible reduction from directed and collaborative MAPF problems to STRIPS\(_{1}^{1}\) and Petri nets, where the tokens are represented by the $|V|$ agents and the inital states and goal states of the agents defined by $s_I$ and $s_G$. 

There are different formuations of collaborative MAPF. In our case Collaboration is modeled by the concept of supporting nodes for edge crossings. 

The literals are the graph nodes, and the edges are defined by the actions. The preconditions of the actions are modeled as collaboration nodes $\operatorname{pre(a)}$, required for moves from $\overline{\operatorname{eff}(a)}$ to $\operatorname{eff}(a)$. Although Nebel\citeyear{Nebel2024-NEBTCC} has shown that the small solution hypothesis holds for directed MAPF, this does not easily transfer to STRIPS\(_{1}^{1}\). 

Note that solutions produced for MAPF are inherently parallel, while the original semantics for STRIPS was sequential, resulting in a totally ordered plan. Recall that some planners like SATplan~\cite{satplan} produce parallel optimal plans, while others like SymBA*~\cite{symba} produce optimal sequential plans.

We implemented a collaborative  MAPF solver based on conflict-based search (CBS) for validating this transformation. We solved such Coop(Di)MAPF with a variant of CBS. The solution below is working, and  smaller than shortest STRIPS$^1_1 $ plan, due to agents moving at the same time (time steps from left to right, node ID leading and assignment to variables in brackets): \medskip 
\\
\noindent
\begin{small}
Agent 0: 0(0) 4(1) 4(1) 4(1) 4(1) 4(1) 0(0) 0(0) 0(0) 4(1) 0(0) 4(1)\\
Agent 1: 1(0) 1(0) 5(1) 5(1) 5(1) 5(1) 5(1) 1(0) 1(0) 1(0) 5(1) 1(0) \\
Agent 2: 2(0) 2(0) 2(0) 6(1) 6(1) 6(1) 6(1) 6(1) 2(0) 2(0) 2(0) 6(1)\\Agent 3: 3(0) 3(0) 3(0) 3(0) 7(1) 7(1) 7(1) 7(1) 7(1) 7(1) 7(1) 7(1) \\
\end{small}

See Fig.~\ref{fig:mapping} for an example illustrating the bijection to collaborative directed MAPF and conservative safe Petri nets.

\section{Conclusion}

The computational complexity for fragments of planning problems like STRIPS task is a fascinating fundamental topic with applications in Petri-net and MAPF theory. 

From different angles, we looked at STRIPS with one precondition and one effect, which was the most challenging of remaining open problems in that class. 
The explicit-state and SAT solver results for small instances and the limits in shortest path length for larger instances to cover the entire cube, indicated
that the solution lengths might be polynomial and the small solution hypothesis for STRIPS$^1_1$ to be true.

While in this paper we are closing the gap, we consider the problem of the NP-completeness of the STRIPS$^1_1$ still to be open. 
%\section{Acknowlegements}

%We thank the world's STRIPS organization for delivering polynomial plans \ldots [to be adapted]

\vfill 

\pagebreak

\bibliography{main}

%\end{document}

\vfill 

\pagebreak

\section*{Appendix: A Failed Proof Attempt by an LLM}

%We need to show that for any STRIPS\(^1_1\) planning instance the length of an optimal plan is bounded by a polynomial. 

%In the following we present what the AI came up with with only minor edits for readability. Although being impressed about the competent answer, we judge following derivations as incomplete, but as a help to envision a proof.

We gave the above text with the prompt to complete the proof to DeepSeek on its web interface. To query a state-of-the-art LLM known to be competent in math tasks for delivering a proof to a challenging problem that some researchers tackled unsuccessfully before, clearly is an unusual way to attack a computational complexity problem, and surely first in AI planning. We did not expect any good result. But we were impressed by the competent answer of the DeepSeek (v3.2)'s on the blunt question to complete the proof based on the derivations (provided as a PDF) we had so far. The model, DeepSeekMath-V2, scored 118 out of 120 points on questions from the 2024 William Lowell Putnam Mathematical Competition, beating the top human score of 90. The model also performed at the level of gold-medal winners in the International Mathematical Olympiad (IMO) 2025 and the 2024 China Mathematical Olympiad~\cite{shao2025deepseekmathv2selfverifiablemathematicalreasoning}. 
%The results are described in a preprint posted on arXiv on 27 November.

Although the large language model introduced several important  concepts on its own including a relevant partitioning of the literal graph into SCCs and other valid reasoning steps, we are, however, not  convinced that its proof is complete. It is for sure verbose by introducing a lemma that essentially says that some special subgraphs (namely SCCs) of a graph do not have more edges than the supergraph. To proof this it uses incorrect statements about the tree height in the overlap tree. Example $n=4$ from the paper has treeheight 11 and is a counterexample. Prompt and answer were as follows.
%
%
%As reported in an article on December 4th, 2026 in Nature\footnote{\url{https://www.nature.com/articles/d41586-025-03959-9}}

\color{red}
Can you complete the proof that Strips planning with actions that have only one precondition and one effect can have plan length that are at most polynomial in the number of variables.
%\end{red}

\medskip

%First, the LLM briefly recalled the setting
\color{blue}
\begin{enumerate} 
\item
\( V \) = set of \( n \) propositional variables.
\item A literal is \( v \) or \( \bar{v} \).
\item An action \( a \) has \( \text{pre}(a) \) (single literal) and \( \text{eff}(a) \) (single literal).
\item literal graph \( G \): nodes = literals \( L(V) \), edges = actions \( A \), each \( a \) from \( \text{pre}(a) \) to \( \text{eff}(a) \).
\item A plan \( \omega = \langle a_1, \dots, a_k \rangle \) takes \( s_I \) to \( s_G \).
\end{enumerate}
\color{black}

%and continues then as follows \medskip

\color{blue}
We define for each literal \( l \) the \emph{last achieving action} in the plan as follows.
Let \( t_l \) be the largest index \( i \) such that \( \operatorname{eff}(a_i) = l \), or 0 if \( l \in s_I \) and never changed in plan.
If \( t_l = 0 \) and \( l \notin s_I \), the plan is invalid.
The subplan \( \omega_{l,t_l} \) is the minimal subsequence of \( \omega \) that makes \( l \) true at time \( t_l \). %, computed recursively as stated.

In the examples we observe that subplans for different goal literals share prefixes but eventually diverge and operate on disjoint sets of edges in literal graph (once diverged).

If two goal literals \( l_1 \) and \( l_2 \) have last achieving actions \( a_{i_1} \) and \( a_{i_2} \) with \( i_1 < i_2 \), then \( \omega_{l_1, i_1} \) and \( \omega_{l_2, i_2} \) may share actions in prefix, but after the last shared action, they use disjoint action sets. After divergence, each subplan is responsible for the truth of a unique literal and only affects its own subset of literals until its goal literal is set.
After divergence, a subplan (for goal literal \( l \)) only traverses edges within a strongly connected component (SCC) of the literal graph containing \( l \) and its preconditions in the plan. This is true, because in order to achieve \( l \), the plan must make \( \operatorname{pre}(a) \) true for some action \( a \) with \( \operatorname{eff}(a) = l \), which again needs its precondition, etc. This forms a dependency subgraph which is contained in an SCC reachable from \( s_I \) literals. 
Since \( |L(V)| = 2n \) we have that
\begin{lemma}
Each SCC has at most \( O(n^2) \) edges.   
\end{lemma} 

%{\bf Proof.} 
%Since each variable has at most two literals in an SCC, each SCC has at most \( 2n \) literals, hence \( O(n^2) \) possible edges. % (actions).

We show that even though a subplan may cycle multiple times over its subset of actions, the number of such cycles is limited by the number of other subplans and the number of actions in the subset.
Suppose a subplan repeats a cycle in its SCC multiple times. Each full cycle flips some variable \( v \) at least twice (true $\rightarrow$ false $\rightarrow$ true). But if \( v \) is flipped more than once, those extra flips might be unnecessary unless needed for other subplans’ preconditions.

We can bound the number of total flips of a variable by \( O(n) \) as follows. In a valid plan, flips of \( v \) are alternating \( v \) and \( \bar{v} \).  
Between two flips of \( v \), at least one other subplan must require \( v \) in the opposite polarity for its precondition.  
There are at most \( 2n \) literals that might need \( v \)’s state change for their own achievement.
Thus, each variable is flipped \( O(n) \) times and the total number of actions is \( O(n^2) \) per SCC.
%$\Box$
%\medskip

We now merge subplans. In the overlap tree, each leaf corresponds to a goal literal’s final achievement.
Let \( m \) be the number of goal literals (\( m \le n \)). Then the tree height is \(\le m \) (actually \(\le n \)). At each node, the subplan’s length is bounded by \( O(n^2) \) (actions within an SCC).
If subplans were completely disjoint after divergence, total length would be \(O(n^3) \).
But subplans share prefixes (common part), so total plan length is at most \(O(n^3) \). 

We can formalize this by induction on the number of goal literals not yet achieved in the suffix of the plan. Define \( \operatorname{Active}(i) \) as the set of goal literals whose last achieving action index  \(\ge i \). As \( i \) decreases from \( k \) to 1, \( \operatorname{Active}(i) \) grows. Each time \( \operatorname{Active}(i) \) changes, a subplan for a new goal literal starts. Between changes to \( \operatorname{Active}(i) \), the plan is working within the SCC of the currently active goal literals’ subplans, hence \( O(n^2) \) actions per segment.
The number of segments is at most \( n \) (each goal literal added once). Thus, \( k \le O(n^3) \).

Can we have a cycle on a set of variables that are all eventually false?  In STRIPS\(^1_1\), if all literals in a cycle are false at end, that cycle is irrelevant to goal state, so can be removed from plan. Thus, optimal plan has no such irrelevant cycles.

Also, can preconditions in a cycle remain false indefinitely? If a precondition of an action in cycle is false forever, that action never becomes applicable, so cycle not executed. So in executed plan, preconditions become true when needed.
Hence, any cycle executed must serve to keep some literal true for some subplan’s need later. This need is bounded by the number of subplans \( O(n) \), so that the number of cycle repetitions is bounded. 

\begin{theorem} 
For any STRIPS\(^1_1\) planning instance with \( n \) variables, the length of an optimal plan is \( O(n^3) \)
\end{theorem}

{\bf Proof.} We combine the results.
as the literal graph SCC size is at most \( O(n^2) \), each variable flips at most \( O(n) \) times, and the overlap tree has \( O(n) \) segments, each segment \( O(n^2) \) actions.  
Thus, the total plan length is bounded by \( O(n^3) \). $\Box$

\medskip

\end{document}